\documentclass{article}

\usepackage{microtype}
\usepackage{graphicx}
\usepackage{subfigure}
\usepackage{booktabs} %

\usepackage{hyperref}

\usepackage[accepted]{icml2023}

\usepackage{amsmath}
\usepackage{amssymb}
\usepackage{mathtools}
\usepackage{amsthm}
\usepackage{nicefrac}
\usepackage{paralist}

\def\R{{\mathbb R}}
\def\E{{\mathbb E}}
\def\Pbb{{\mathbb P}}

\def\Xcal{{\mathcal X}}
\def\Pcal{{\mathcal P}}

\def\Hcal{{\mathcal H}}

\def\Gcal{{\mathcal G}}
\def\Vcal{{\mathcal V}}
\def\Ecal{{\mathcal E}}
\def\Scal{{\mathcal S}}

\def\Ycal{{\mathcal Y}}
\def\Dcal{{\mathcal D}}

\def\x{{\mathbf x}}
\def\z{{\mathbf z}}

\def\Mcal{{\mathcal M}}
\def\Ncal{{\mathcal N}}

\def\id{{\mathrm{id}}}

\def\KL{{\mathrm{KL}}}

\DeclareMathOperator*{\argmin}{arg\,min}
\DeclareMathOperator*{\argmax}{arg\,max}
\newcommand{\ie}{\textit{i.e., }}
\newcommand{\eg}{\textit{e.g., }}

\newcommand{\norm}[1]{\left\lVert#1\right\rVert}

\usepackage[capitalize,noabbrev]{cleveref}

\usepackage{tikz}
\usepackage{pgfplots}
\usetikzlibrary{calc}

\newcommand\DrawDot[3]{
  \draw node[#2,circle,fill=#2,inner sep=1.2pt,%
  label={[black]-45:{#3}}] at #1 {};
}

\theoremstyle{plain}
\newtheorem{theorem}{Theorem}[section]

\newtheorem{lemma}[theorem]{Lemma}

\theoremstyle{definition}

\theoremstyle{remark}

\usepackage[textsize=tiny]{todonotes}

\icmltitlerunning{Fisher Information Embedding for Node and Graph Learning}

\begin{document}

\twocolumn[
\icmltitle{Fisher Information Embedding for Node and Graph Learning}

\icmlsetsymbol{equal}{*}

\begin{icmlauthorlist}
\icmlauthor{Dexiong Chen}{equal,eth,sib}
\icmlauthor{Paolo Pellizzoni}{equal,eth,sib}
\icmlauthor{Karsten Borgwardt}{eth,sib}
\end{icmlauthorlist}

\icmlaffiliation{eth}{Department of Biosystems Science and Engineering, ETH Zürich, Switzerland}
\icmlaffiliation{sib}{SIB Swiss Institute of Bioinformatics, Switzerland}

\icmlcorrespondingauthor{Dexiong Chen}{dechen@ethz.ch}
\icmlcorrespondingauthor{Paolo Pellizzoni}{ppellizzoni@ethz.ch}

\icmlkeywords{graph neural networks, graph kernel, node embedding, set pooling, attention mechanism}

\vskip 0.3in
]

\printAffiliationsAndNotice{\icmlEqualContribution} %

\begin{abstract}
Attention-based graph neural networks (GNNs), such as graph attention networks (GATs), have become popular neural architectures for processing graph-structured data and learning node embeddings. Despite their empirical success, these models rely on labeled data and the theoretical properties of these models have yet to be fully understood. In this work, we propose a novel attention-based node embedding framework for graphs. Our framework builds upon a hierarchical kernel for multisets of subgraphs around nodes (\eg neighborhoods) and each kernel leverages the geometry of a smooth statistical manifold to compare pairs of multisets, by ``projecting'' the multisets onto the manifold. By explicitly computing node embeddings with a manifold of Gaussian mixtures, our method leads to a new attention mechanism for neighborhood aggregation. We provide theoretical insights into generalizability and expressivity of our embeddings, contributing to a deeper understanding of attention-based GNNs. We propose both efficient unsupervised and supervised methods for learning the embeddings. Through experiments on several node classification benchmarks, we demonstrate that our proposed method outperforms existing attention-based graph models like GATs. Our code is available at \url{https://github.com/BorgwardtLab/fisher_information_embedding}.
\end{abstract}

\section{Introduction}

Graph-structured data has gained significant attention in various domains in recent years. To analyze this type of data with graph relationships, graph neural networks~(GNNs) have been widely used and have been shown to be powerful tools. They have been successfully applied to broad classes of application domains, such as drug discovery~\citep{gaudelet2021utilizing}, protein design~\citep{ingraham2019generative} and social network analysis~\citep{fan2019graph}. 

GNNs construct multilayer models and iteratively perform neighborhood aggregation on previous layers to generate new node representations~\citep{xu2018powerful}. A popular variant of GNNs, known as graph attention networks (GATs)~\citep{velickovic2018graph}, utilizes an attention mechanism for neighborhood aggregation. The attention mechanism of GATs generalizes traditional average and max pooling of the representations of neighbors by performing a data-dependent weighted average of neighbors. Despite the empirical success of GATs, their theoretical properties have not been fully explored. A deeper understanding of the generalizability and expressivity of attention-based GNNs is crucial for identifying limitations and developing more advanced models. Additionally, GATs rely on labeled data, making them less flexible than unsupervised node embedding methods, \eg these derived from trainable kernel methods~\citep{chen2020convolutional}.

Another important class of methods for dealing with the complexity of graph-structured data is node and graph kernels. Node kernels~\citep{kondor2004diffusion,smola2003kernels} have demonstrated success in capturing the similarity between two nodes in a graph by leveraging the graph Laplacian. However, this class of kernels are known to possess limitations, including high computational complexity~\citep{smola2003kernels} and challenges in generalizing across multiple graphs. In contrast, graph kernels~\citep{borgwardt2020graph,kriege2020survey} rely on the $\mathcal{R}$-convolution framework, which computes local similarities among substructures and aggregates these similarities. Recent work has also identified limitations of this framework and proposed approaches to overcome them. For example, \citet{togninalli2019wasserstein} leverage optimal transport to aggregate node embeddings obtained by the Weisfeiler-Lehman~(WL) process, in order to improve their ability to capture complex characteristics of the graph. This kernel was later extended to obtain graph embeddings that can handle large-scale datasets~\citep{kolouri2020wasserstein}. However, these kernels or embeddings are based on unparameterized node embeddings obtained by averaging the node representations of neighbors, without any non-linearities. This can restrict their ability to accurately predict node properties in certain tasks.

In this work, we propose a novel class of attention-based node embeddings to address the above limitations of existing attention-based graph models. Our method approximates multisets of node features with a class of smooth probability distributions and compares pairs of distributions using a statistical divergence, specifically the Kullback–Leibler~(KL) divergence. By leveraging the geometry of the manifold of the smooth distributions, we demonstrate that the divergence can be approximated, locally at an anchor distribution, by an embedding distance. Additionally, we show that the embeddings can be computed explicitly for a manifold of Gaussian mixtures and that the resulting embedding of each node leads to a new attention mechanism, providing theoretical insights into the expressivity and generalizability of this class of node embeddings. We propose efficient unsupervised and supervised methods for learning the node embeddings, using the same architecture. We validate our theoretical findings on synthetic datasets. Our empirical evaluation further shows that our proposed node embeddings can achieve similar or even better performance compared to GATs on several node classification benchmarks. To summarize, our key contributions are:
\begin{itemize}
    \item We introduce a new framework for node representation learning based on hierarchical kernels for multisets.
    \item We propose a new general embedding, named Fisher information embedding (FIE), for probability measures and multisets by approximating the KL divergence on a smooth statistical manifold. We further provide theoretical insights into the generalizability and expressivity of our embeddings, contributing to a deeper understanding of attention-based GNNs.
    \item When restricting the analysis to Gaussian mixtures, we demonstrate that FIE exhibits a closed form and can be learned efficiently. We provide both unsupervised and supervised methods for learning the node embeddings.
    \item Finally, we validate our theoretical findings on both simulated and real-world data and demonstrate that our methods achieve comparable or better results to existing attention-based GNNs.
\end{itemize}

We will present related work in Section~\ref{sec:related_work}.
In Section~\ref{sec:framework}, we introduce a general multilayer kernel embedding for nodes where each kernel operates on multisets of node features (\eg neighborhoods of nodes). In Section~\ref{sec:fie}, we show how to define such a kernel for multisets, by transforming multisets to probability distributions through maximum likelihood estimation and comparing probability distributions through a kernel defined on the manifold of parametric distributions. In Section~\ref{sec:experiments}, we validate the effectiveness of our node embeddings on both simulated and real-world datasets.

\section{Related work}\label{sec:related_work}
This paper focuses on node embedding methods and general set pooling methods. We present and discuss in this section recent work in both fields.
\paragraph{Node embedding methods}
Our method provides an unsupervised node embedding approach, which is a popular and well-established method for graph-structured data. Some well-known unsupervised node embedding methods include DeepWalk~\citep{perozzi2014deepwalk} and node2vec~\citep{grover2016node2vec}, which extract node representations by performing random walks on the graph. In \citet{rozemberczki2021multi}, information from random walks at different scales was pooled to produce the embeddings. Additionally, in \citet{abu2018watch}, the authors proposed a node embedding method with attention-driven random walks. The method presented in \citet{xu2019gromov} proposed structural node embeddings that were learned by matching two distinct graphs using the optimal transport framework. In \citet{zhu2021node}, the authors presented a node embedding that mapped nodes to their embeddings according to both positional proximity and structural similarity criteria. A simple yet effective node embedding can be extracted using the WL algorithm for graph isomorphism~\citep{shervashidze2011weisfeiler, togninalli2019wasserstein}, which used a message-passing framework to refine the either categorical or continuous node attributes. \citet{chen2020convolutional} proposed a kernel embedding for nodes based on sum aggregations of path features. In contrast, our approach defines a kernel embedding for arbitrary multisets beyond simple sum aggregation, including multisets of path features by~\citet{chen2020convolutional}.
\paragraph{Set pooling methods}
One of the essential components in message-passing-based graph learning methods is pooling information from node neighborhoods, which often involves pooling a (multi)set of features. One of the first methods to provide an encoding for sets was DeepSets~\citep{zaheer2017deep}, which transformed the elements of the set with a neural network and pooled them using either their sum or average. Other methods include optimal transport-based pooling~\citep{mialon2020trainable,kim2021differentiable} and trainable set embeddings based on the sliced-Wasserstein distance~\citep{naderializadeh2021pooling}.
Set pooling is also used to provide graph-level embeddings, typically by pooling node embeddings. For example, WEGL~\citep{kolouri2020wasserstein} obtained a graph embedding by representing the set of node embeddings as a probability distribution and computing an embedding based on the Wasserstein distance to a reference distribution. Additionally, in~\citet{cheng2022revisiting}, the authors unified some existing global pooling methods, such as mean, max, and attention pooling, under the framework of optimal transport by finding an optimal transport plan between the samples of the sets and the features of the embeddings.

\section{A general framework for node learning}\label{sec:framework}
In this section, we present our node embedding framework for graphs, discuss its link to previous work and provide insights into its expressivity and generalizability.

\subsection{Multisets and kernels for multisets}\label{sec:multisets}
A multiset is a generalized notion of a set that allows multiple instances of its elements. By abuse of notation, we will use the same notation as for a set. Here, we handle multisets of features in $\R^d$ living in
\begin{equation*}
\begin{split}
    \Xcal^d = \Bigl\{ \x ~|~ \x = \{\x_1, \dots, \x_n\} \text{ such that } \\ \x_1, \dots, \x_n \in \R^d \text{ for some }n \geq 1 \Bigr\}.
\end{split}
\end{equation*}
The cardinality of a multiset, denoted by $|\cdot|$ is defined as the number of elements by summing the multiplicities. 
Here, we assume that we have access to a kernel defined on the space of multisets $K_{\mathrm{ms}}:\Xcal^d\times\Xcal^d\to\R$ and its exact or approximate embedding $\psi_{\mathrm{ms}}:\Xcal^d\to\R^p$ such that $K_{\mathrm{ms}}(\x,\x')\approx\langle\psi_{\mathrm{ms}}(\x),\psi_{\mathrm{ms}}(\x')\rangle$ through, \eg the Nystr\"om approximation~\citep{williams2000using}. We will describe in Section~\ref{sec:fie_multisets} how to define such a kernel on the space of multisets. It is worth noting that $K_{\mathrm{ms}}$ can be the composition of multiple kernels, and the resulting embedding $\psi_{\mathrm{ms}}$ can be obtained by composing the kernel mappings.

\subsection{A multilayer kernel for nodes and graphs}
Let us denote by $\Gcal=(\Vcal, \Ecal, a)$ a graph with vertices $\Vcal$ and edges $\Ecal$ associated with node attributes $a:\Vcal\to\R^{d_0}$. For any node $v\in\Vcal$, we denote by $\Scal_{\Gcal}(v)$ a certain multiset of subgraphs in $\Scal$ rooted at $v$ in $\Gcal$ and assume that there exists an \emph{injective} mapping $h$ from $\Scal$ to $\R^{d_{\nicefrac{1}{2}}}$ that represents any subgraph in $\Scal$ as a $d_{\nicefrac{1}{2}}$-dimensional vector where $d_{\nicefrac{1}{2}}$ denotes some hidden dimension. We denote by $h(\Scal_{\Gcal}(v))$ the multisets in $\Xcal^{d_{\nicefrac{1}{2}}}$ consisting of the images of any subgraph in $\Scal_{\Gcal}(v)$ under $h$. Then, for any two nodes $v$ and $v'$ in two graphs $\Gcal$ and $\Gcal'$, we can define a class of kernels based on the previously defined multiset kernel:
\begin{equation}
	K^{(1)}(v,v'):=K_{\mathrm{ms}}^{(1)}\left(h(\Scal_{\Gcal}(v)),h(\Scal_{\Gcal'}(v'))\right).
\end{equation}
An (approximate) embedding for the above kernel is given by $\psi_1(v):=\psi^{(1)}_{\mathrm{ms}}(h(\Scal_{\Gcal}(v)))\in\R^{d_1}$. This results in a new graph with the same set of vertices and edges but a different attribute function $\Gcal_1=(\Vcal,\Ecal, \psi_1:\Vcal\to\R^{d_1})$. Repeating this process $T$ times results in a sequence of graphs $\Gcal_1,\Gcal_2,\dots,\Gcal_{T}$, where each graph $\Gcal_t$ carries a new node embedding $\psi_t:\Vcal\to\R^{d_t}$. A final graph-level embedding can also be obtained using a new multiset kernel, by viewing a graph as a multiset of node embeddings. Finally, a multilayer kernel between two nodes can be defined as:
\begin{equation}\label{eq:multilayer_kernel}
    K(v,v'):=\sum_{t=0}^T \langle\psi_t(v),\psi_t(v') \rangle_{\R^{d_{t}}},
\end{equation}
where $\psi_0=a$. This class of kernels includes several previous graph kernels such as \citet{chen2020convolutional} and \citet{morris2022speqnets}.

\subsubsection{Examples of $\Scal_{\Gcal}$ and $h$}\label{sec:examples}
\paragraph{Neighborhoods}
Similar to GNNs, a straightforward choice of $\Scal_{\Gcal}(v)$ is the neighborhood of $v$, \ie $\Scal_{\Gcal}(v):=\Ncal_{\Gcal}(v)$ or the neighborhood including the root node $\Scal_{\Gcal}(v):=\Ncal_{\Gcal}(v)\cup \{v\}$. $h_t$ at layer $t\geq 1$ can be simply chosen as the node embedding from the previous layer $h_t:=\psi_{t-1}$. Alternatively, one could also consider the 2-tuples consisting of $v$ and its neighbors, namely $\Scal_{\Gcal}(v):=\{(v,u):u\in\Ncal_{\Gcal}(v) \}$, and $h_t(v, u):=\nicefrac{1}{2}(\psi_{t-1}(v)+\psi_{t-1}(u))$ as used in~\citet{togninalli2019wasserstein} and \citet{kolouri2020wasserstein}. 

\paragraph{Paths and higher order multisets}
Similar to \citet{chen2020convolutional}, one could also use paths to define the set of subgraphs. Specifically, let us denote by $\Pcal_{\Gcal}^k(v)$ the paths from $v$ of fixed length $k$. Then, we define $\Scal_{\Gcal}(v):=\Pcal_{\Gcal}^k(v)$ and $h_t(p):=\psi_{t-1}(p)$ for any path $p$ as the concatenation of the node features in this path. For higher order multisets defined by the $k$-dimensional WL algorithms~\citep{morris2022speqnets}, one could define $\Scal_{\Gcal}$ and $h$ in a similar way as in previous work. 

In this work, we only focus on the above example of neighborhoods which has a tight link with GATs. We illustrate this step via the arrow (a) of Figure~\ref{fig:overview}.

\subsubsection{Theoretical properties}
Here, we discuss the necessary conditions for defining a good kernel embedding on multisets and will present in Section~\ref{sec:fie} an embedding satisfying these conditions.
A good kernel embedding $\psi_{\mathrm{ms}}$ should be injective to guarantee the expressive capability of the node embeddings. This injectivity assumption was also widely adopted previously~\citep{xu2018powerful} to prove the expressivity of GNNs. In particular when we consider the neighborhoods example for $\Scal_{\Gcal}$ and $h$, we can show the following link between our embedding and WL test by using similar arguments to \citet[Theorem 3]{xu2018powerful}
\begin{lemma}\label{lemma:expressive}
    If $\psi^{(t)}_{\mathrm{ms}}$ is injective for all $t$ in $1,\dots,T$, then $\psi_t(v)\neq \psi_t(v')$ for any $v$ and $v'$ that the WL isomorphism test decides as non-isomorphic.
\end{lemma}
In addition to expressivity, a good kernel should also guarantee the stability and invariance of the predictions, which is controlled by the induced reproducing kernel Hilbert space~(RKHS) norm. The RKHS norm also provides a natural way to control the model complexity, leading to generalization bounds through, \eg Rademacher complexity and margin bounds~\citep{boucheron2005theory,shalev2014understanding}. Specifically, the generalization bound is given by the following classical result:
\begin{lemma}[\citet{boucheron2005theory}]\label{lemma:generalization}
    Consider a binary task with labels in $\Ycal=\{-1,1\}$ and nodes on a graph $\Gcal=(\Vcal,\Ecal,a)$. Let us denote by $\Dcal$ the data distribution of samples from $\Dcal$ and $\Ycal$. We define $L(f):=\Pbb_{(v,y)\sim \Dcal}(y f(v)<0)$ as the expected error of a function $f:\Vcal\to \R$. For a training dataset $(v_1,y_1),\dots,(v_N,y_N)$, we define the training error with confidence margin $\gamma>0$ as $L^{\gamma}_N(f):=\nicefrac{1}{N}\sum_{i=1}^N \mathbf{1}_{y_i f(x_i)<\gamma}$. Then, with probability $1-\delta$, we have, for any $\gamma>0$ and $f\in\Hcal_{K}$ the RKHS of $K$ in \eqref{eq:multilayer_kernel}
    \begin{equation}
        L(f)\leq L^{\gamma}_N(f)+O\left(\frac{\bar{B}\norm{f}_{\Hcal_K}}{\gamma N} + \sqrt{\frac{\log(\nicefrac{1}{\delta})}{N}}\right),
    \end{equation}
    where $\bar{B}=\sqrt{\nicefrac{1}{N}\sum_{i=1}^N K(v_i,v_i)}$.
\end{lemma}
However, the RKHS norm of $f$ is generally unknown and depends on the regularity of $\psi_t$ which is hard to characterize. In section~\ref{sec:theory_fie}, we will show weak results that implicitly control the RKHS norm of $f$. A more comprehensive theoretical analysis of the generalization bounds will be the subject of future work.

\section{Fisher information embedding for multisets}\label{sec:fie}
\begin{figure*}[t]
    \centering
    \resizebox{\textwidth}{!}{\input{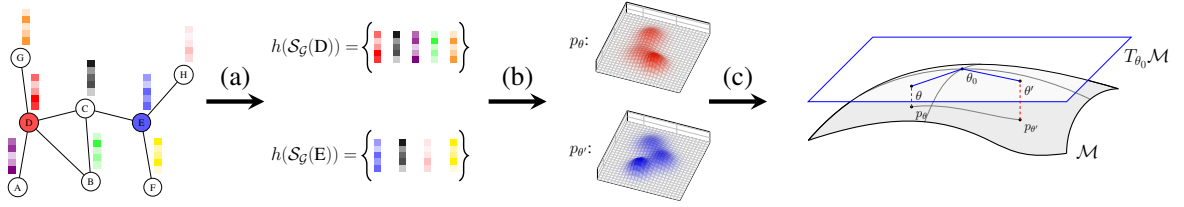}}
    \caption{An illustration of the Fisher Information Embedding for nodes. (a) Multisets $h(\mathcal{S_G}(\cdot))$ of node features are obtained from the neighborhoods of each node. (b) Multisets are transformed to parametric distributions, \eg $p_\theta$ and $p_{\theta'}$, via maximum likelihood estimation. (c) The node embeddings are obtained by estimating the parameter of each distribution using the EM algorithm at an anchor distribution $p_{\theta_0}$ as the starting point.
    The last panel shows a representation of the parametric distribution manifold $\mathcal{M}$ and its tangent space $T_{\theta_0}\mathcal{M}$ at the anchor point $\theta_0$. The points $p_{\theta}$ and $p_{\theta'}$ represent probability distributions on $\mathcal{M}$ and the gray dashed line between them their geodesic distance. The red dashed lines represent the retraction mapping $R_{\theta_0}^{-1}$.}
    \label{fig:overview}
\end{figure*}

In this section, we present a kernel for probability measures and multisets. We prove that the associated kernel embedding is a local approximation of the KL divergence, providing theoretical insights into the generalizability and expressivity of the embeddings. For a manifold of Gaussian mixtures, we show that the induced kernel embedding can be computed explicitly and results in a new class of attention-based node embeddings.

\subsection{A tangent-based kernel for probability measures}
Let us consider a smooth statistical manifold $\Mcal$ defined as a subset of probability measures. We assume that $\Mcal$ is a Riemannian manifold and denote by $T_{\mu}\Mcal$ its tangent space at every point $\mu\in\Mcal$ endowed with a positive definite inner product $g_{\mu}:T_{\mu}\Mcal\times T_{\mu}\Mcal\to\R$, denoted by $\langle\,,\,\rangle_{g_{\mu}}$. We also assume that there exists a retraction mapping $R_{\mu}:T_{\mu}\Mcal\to \Mcal$ on each point $\mu\in\Mcal$ satisfying $R_{\mu}(0)=\mu$ and $\mathrm{D}R_{\mu}(0)=\id_{T_{\mu}\Mcal}$. By the inverse function theorem, $R_{\mu}$ is a local diffeomorphism and we denote by $R_{\mu}^{-1}$ its local inverse. Now at every point $\mu$, we can define a positive definite kernel on the proximity of $\mu$ in $\Mcal$:
\begin{equation}\label{eq:tangent_kernel}
	K_{\mu}(u,v)=\langle R_{\mu}^{-1}(u), R_{\mu}^{-1}(v)\rangle_{g_{\mu}}.
\end{equation}
By using the mathematical tools from information geometry, we show that this kernel approximates well any divergence associated to the Riemannian metric $g$, which is given by the following theorem adapted from~\citet[Theorem 3.20]{amari2000methods}:
\begin{theorem}\label{thm:information_geometry}
	Let $D$ be any divergence. We have
	\begin{equation}
		D(u\|\mu)+D(\mu\| v)-D(u\| v)=K_{\mu}(u,v)+O(\Delta^2),
	\end{equation}
	where $\Delta:=\max\{\|\xi(u)-\xi(\mu) \|, \|\xi(v)-\xi(\mu)\|\}$ for an arbitrary coordinate system $\xi$.
\end{theorem}
The above kernel is a general tangent-based kernel, generalizing several embeddings proposed in the context of optimal transport, such as \citet{mialon2020trainable} and \citet{kolouri2020wasserstein}. In the following section, we will focus on an example of this kernel that exhibits a closed form, which can be easily used in practice.

\subsection{Fisher information embedding for parametric probability distributions and multisets}\label{sec:fie_multisets}
Here, we consider a parametric family of probability distributions $\Mcal=\{p_{\theta}(x)\}_{\theta\in\Theta}$ with $\Theta\subset \R^{m}$ and $\theta$ denoting a global coordinate system of $\Mcal$. A Riemannian metric defined by the Fisher information metric at every point $p_{\theta_0}$ is given by:
\begin{equation}
	\langle \theta,\theta'\rangle_{g_{\theta_0}}:=\theta^{\top}I(\theta_0)\theta'~~\text{for any}~~\theta,\theta'\in T_{\theta_0}\Mcal=\Theta,
\end{equation}
where $I(\theta_0)$ denotes the Fisher information matrix given by
\begin{equation}
	I(\theta_0):=\E[\nabla_{\theta}\log p_{\theta}(x)|_{\theta=\theta_0}\nabla_{\theta}\log p_{\theta}(x)|_{\theta=\theta_0}^{\top}]\in\R^{m\times m}.
\end{equation}
This symmetric matrix is positive definite and is the Hessian matrix of the KL divergence:
\begin{lemma}\label{lemma:kl_div}
	Let us denote by $D(\theta_0,\theta):=\KL(p_{\theta_0}\|p_{\theta} )=\int p_{\theta_0}(x)\log\nicefrac{p_{\theta_0}(x)}{p_{\theta}(x)}\mathrm{d}x$. If $\theta$ is close to $\theta_0$, we have
	\begin{equation}
		D(\theta_0,\theta)=\frac{1}{2}(\theta-\theta_0 )^{\top}I(\theta_0)(\theta-\theta_0)+o((\theta-\theta_0)^2).
	\end{equation}
\end{lemma}
The retraction mapping at $\theta_0$ is simply defined via the parametric coordinates: $R_{\theta_0}:\theta\mapsto p_{\theta+\theta_0}(x)\in\Mcal$. Since $R_{\theta_0}$ is invertible on the entire $\Theta$, the tangent-based kernel in Equation~\eqref{eq:tangent_kernel} for this parametric family is defined on $\Mcal\times\Mcal$ and is given by:
\begin{equation}\label{eq:kernel_proba}
	K_{\theta_0}(p_{\theta},p_{\theta'})=(\theta-\theta_0)^{\top} I(\theta_0)(\theta'-\theta_0).
\end{equation}
And the \emph{Fisher information embedding (FIE) for probability distributions} is defined as:
\begin{equation}\label{eq:fie_proba}
    \varphi_{\theta_0}(p_{\theta}):=I(\theta_0)^{\nicefrac{1}{2} }(\theta-\theta_0).
\end{equation}
However, when comparing two multisets, $\theta$ and $\theta'$ generally are not known but only samples $\x$ and $\x'$ respectively from the corresponding distributions are known. Since the Fisher information metric locally behaves similarly to the KL divergence one could ``project'' the distribution of $\x$ onto $\Mcal$ by finding a probability distribution in $\Mcal$ that minimizes its divergence with $p_{\x}$, where $p_{\x}$ denotes the true density of $\x$. This amounts to computing the maximum log-likelihood (ML) estimator:
\begin{equation} \label{eq:ml}
\begin{split}
    \theta_{\mathrm{ML}}(\x) &:=\argmin_{\theta\in\Mcal} \KL(p_{\x}\| p_{\theta}) \\
    &=\argmax_{\theta\in\Mcal}\E_{x\sim p_{\x}(x)}[\log p_{\theta}(x)].
\end{split}
\end{equation}
For any multisets $\x$ and $\x'$ in $\Xcal^d$, we assume that their ML estimators are accessible and denoted by $\theta_{\mathrm{ML}}(\x)$ and $\theta_{\mathrm{ML}}(\x')$ respectively. Then, similar to kernel in Eq.~\eqref{eq:kernel_proba}, we can define a kernel for multisets:
\begin{equation}
	K_{\theta_0}(\x,\x')=(\theta_{\mathrm{ML}}(\x)-\theta_0)^{\top} I(\theta_0)(\theta_{\mathrm{ML}}(\x')-\theta_0),
\end{equation}
associated with the kernel mapping 
\begin{equation}\label{eq:fie_multisets}
\varphi_{\theta_0}(\x):=I(\theta_0)^{\nicefrac{1}{2} }(\theta_{\mathrm{ML}}(\x)-\theta_0),
\end{equation}
which we call the \emph{Fisher information embedding for multisets}. We illustrate both ML estimation and comparing probability distributions through step (b) and (c) in Figure~\ref{fig:overview}.

However, ML estimators do not necessarily have a closed-form solution in general and could still be hard to compute in practice. Fortunately, if the probability family can be described more simply with an additional latent variable $z$ as $p_{\theta}(x,z)$, one can instead consider the following simpler estimator dependent of $\theta_0$:
\begin{equation}
	\theta_{\mathrm{ML}_{\theta_0}}(\x):= \argmax_{\theta\in\Mcal}\E_{z|x,\theta_0}[\log p_{\theta}(x,z)].
\end{equation}
This estimator can be seen as a good approximation of the ML estimator (their relationship is given in the Appendix) and can be computed by performing a single iteration of the expectation maximization (EM) algorithm using the initial parameter $\theta_0$. More EM iterations could be performed to approximate better the ML estimator. In Section~\ref{sec:gaussian-case}, we will show how to explicitly compute this estimator when restricted to a manifold of Gaussian mixtures.

\subsection{Theoretical analysis of the embedding}\label{sec:theory_fie}
We characterize the Fisher information embedding by linking it to the KL divergence. For any anchor parameter $\theta_0\in\Theta$ and any pair of parameters $\theta, \theta' \in \Theta$, we show that the $\ell_2$-distance between the Fisher information embeddings of $p_{\theta}$ and $\theta'$ approximates the KL divergence locally. Specifically,
\begin{theorem}\label{thm:closeness}
We have
\begin{equation}
    D(p_\theta\| p_{\theta'}) = \frac{\norm{\varphi_{\theta_0}(p_{\theta}) - \varphi_{\theta_0}(p_{\theta'})}^2}{2} + O(\Delta^2),
\end{equation}
where $\Delta=\max\{\|\theta-\theta_0\|,\|\theta'-\theta_0\|\}$.
\end{theorem}
We also show that if the ML estimators in \eqref{eq:ml} satisfy some conditions, the FIE for multisets is injective:
\begin{theorem}\label{thm:injectivity}
$\varphi_{\theta_0}$ in~\eqref{eq:fie_proba} is injective in $\Theta$. If $\theta_{\mathrm{ML}}$ defined in~\eqref{eq:ml} is injective, then $\varphi_{\theta_0}$ in~\eqref{eq:fie_multisets} is also injective in $\Xcal^d$.
\end{theorem}
This lemma combined with Lemma~\ref{lemma:expressive} shows that the FIE can be as expressive as the WL isomorphism test if the ML estimator is good enough.
We can also show that the FIE for probability distributions is Lipschitz:
\begin{theorem}\label{thm:lipschitz}
$\varphi_{\theta_0}$ in~\eqref{eq:fie_proba} is $\norm{I(\theta_0)^{\nicefrac{1}{2}}}_2$-Lipschitz on $\Theta$ such that
\begin{equation*}
    \norm{\varphi_{\theta_0}(p_{\theta})-\varphi_{\theta_0}(p_{\theta'})}_2\leq \norm{I(\theta_0)^{\nicefrac{1}{2}}}_2\norm{\theta-\theta'}_2.
\end{equation*}
\end{theorem}
The Lipschitz constant plays an important role in deriving the generalization bounds of both deep networks and deep kernels, as studied in~\citet{bartlett2017spectrally,neyshabur2018pac} and the above theorem offers some insights into the generalization properties of FIE. A more complete study of the generalization bounds of FIE will be left for future work.

\subsection{Fisher information embedding induced by the manifold of Gaussian mixtures} \label{sec:gaussian-case}
In this section, we specialize our general framework to a particular statistical manifold.
Let us consider a simple Gaussian mixture as the parametric family, given by
\begin{equation}
	p_{\theta}(x):= \frac{1}{p} \sum_{j=1}^p \Ncal(x;\mu_{j}, I), \label{eq:gaussian-mixt}
\end{equation}
where we only assume $\mu_j$ to be parameters and $\theta:=\{\mu_1,\dots,\mu_p\}$. We also assume that $\theta_0=\{w_0,\dots,w_p\}$. Let $\z=\{z_1,\dots,z_n\}$ be the latent variables that determine the component from which the observation originates such that $p_{\theta}(x_i)=\sum_{j=1}^p p_{\theta}(x_i,z_i=j)$.
Then, for any multiset $\x$, the corresponding log-likelihood can be lower bounded by the Jensen inequality:
\begin{equation*}
\begin{split}
    \frac{1}{n} \sum_{i=1}^n \log p_{\theta}(x_i) &= \sum_{i=1}^n \log \left( \sum_{j=1}^p \alpha_{ij} \frac{p_{\theta}(x_i,z_i=j)}{\alpha_{ij}} \right) \\
    &\geq \sum_{i=1}^n \sum_{j=1}^p \alpha_{ij}\log\frac{p_{\theta}(x_i,z_i=j)}{\alpha_{ij}},
\end{split}
\end{equation*}
for any $\boldsymbol{\alpha}\in\Pi:=\{\alpha_{ij}\geq 0,\sum_{j}\alpha_{ij}=1 \}$. The EM algorithm consists of the following two steps through a maximization-maximization process:
\begin{itemize}
	\item \textbf{E-step:} This step consists of maximizing the right-handed term with respect to $\alpha$, when fixing $\theta=\theta_0$ leading to:
	\begin{equation*}
		\min_{\boldsymbol{\alpha}\in\Pi}-\log p_{\theta_0}(x_i,z_i=j) - H(\boldsymbol{\alpha}),
	\end{equation*}
	where $H(\boldsymbol{\alpha}):=-\sum_{j}\alpha_{ij}\log\alpha_{ij}$ denotes the entropy. 
        Without further constraints on $\boldsymbol{\alpha}$, one can show that the optimal $a_{ij}$'s are
        \begin{equation*}
            \alpha_{ij}=p_{\theta}(z_i=j|x_i)=\frac{\Ncal(x_i,w_j,\Sigma_j)}{\sum_{l=1}^p \Ncal(x_i,w_l,\Sigma_l)}.
        \end{equation*}
        As argued in the Appendix, one can add additional constraints on $\boldsymbol{\alpha}$, \eg to avoid collapsed solutions. For particular constraints, the optimal $a_{ij}$'s can be found via solving an optimal transport problem. 
        
	\item \textbf{M-step:} This step consists of maximizing the right-handed term with respect to $\theta$. In fact, $\theta$ takes the same form as a weighted MLE for a normal distribution, thus:
	\begin{equation}
		\mu_j=\frac{\sum_i \alpha_{ij}x_i}{\sum_i \alpha_{ij}}.
	\end{equation}
\end{itemize}
Under the assumption that the optimal parameter $\theta_{\mathrm{ML}}$ is close to $\theta_0$, the procedure converges in one or few iterations. 

As shown in \citet{sanchez2013image}, under the condition that the components of the Gaussian mixture are well separated, the Fisher information matrix for the parametric family we consider (Eq. \eqref{eq:gaussian-mixt}) can be approximated by $I(\theta_0) = \frac{1}{p} I$.
An approximate form for the FIE on the manifold of Gaussian mixtures is then given by:
\begin{equation}\label{eq:gmm_fie}
	\varphi_{\theta_0}(\x)=\frac{1}{\sqrt{p}} \left( \mathbin\bigg\Vert_{j=1}^p  \frac{\sum_i \alpha_{ij}x_i}{\sum_i \alpha_{ij}} \ \ - \theta_0 \right).
\end{equation}
It is worth noting that the sum of weights $\nicefrac{\alpha_{ij}}{\sum_i \alpha_{ij}}$ equals 1, resulting in a new attention-like mechanism. 
The parameter of this embedding is the parameter of the anchor distribution $\theta_0$, which can be learned in either unsupervised or supervised way, as described in Section~\ref{sec:learning}. While a similar embedding was explored in~\citet{kim2021differentiable}, our embeddings have a geometric interpretation.
Additionally, the $p$ components can be interpreted as ``heads'', but unlike in GATs, they are not independent. This new form of attention mechanism, as demonstrated in our experimental evaluation in Section~\ref{sec:experiments}, yields comparable classification accuracy to GATs.

We discuss here briefly the complexity: each iteration of the EM algorithm involves computing the $np$ entries of the matrix $\boldsymbol{\alpha}$, where $n$ is the number of elements in $\boldsymbol{x}$ and $p$ is the number of components in the mixture model.
Each entry $\alpha_{ij}$ can be computed in constant time after some pre-processing, which leads to a complexity of $O(np)$ per iteration.

As discussed in Section~\ref{sec:multisets}, the kernel mapping $\psi_{\rm multiset}$ can be obtained via the composition of multiple kernel mappings.  
In such a way, one can add addtional non-linearities to the FIE for multisets, making it reminiscent of the attention mechanisms used in \citet{velickovic2018graph, brody2022how}.
This can be achieved by defining $\psi^{}_{\rm multiset} = \psi \circ \varphi_{\theta_0}$, with $\psi$ a nonlinear function.
For example, one can use $\psi(x) = {\rm ReLU}(W^\top x)$, or the exponential dot-product kernel used in \cite{chen2020convolutional} to obtain an unsupervised embedding.

\subsection{Learning algorithms}\label{sec:learning}
We provide both unsupervised and supervised approaches to learn the main parameters $\theta_0$ in the FIE in \eqref{eq:gmm_fie}. 

\paragraph{Unsupervised learning}
Following the definition, $p_{\theta_0}$ must be close to any data densities $p_{\x}$ and thus well characterize the the context distribution. As a consequence, we can fit the statistical model $p_{\theta_0}$ on the union of the multisets from the training dataset. Particularly in the case of neighborhood multisets in Section~\ref{sec:examples}, $\theta_0$ can be simply learned by fitting the Gaussian mixture model on a subset of node features sampled from the training dataset. The same process can be used to learn $\theta_0$ in layer $t>1$ after computing the node embeddings at layer $t-1$.
Furthermore, we remark that learning the embeddings can be carried out in \emph{a single forward pass}, and is therefore much more efficient than end-to-end supervised learning approaches by GNNs. 
In practice, we observe that k-means, known as a special case of EM algorithms~\citep{lucke2019k}, performs comparably while being less computationally costly and thus we use k-means to learn $\theta_0$ for all the experiments.

\paragraph{Supervised learning}
Similar to previous studies~\citep{chen2020convolutional,kim2021differentiable}, $\theta_0$ can also be viewed as model parameters and can be learned end-to-end by minimizing an objective function for a downstream task. In this case, $\theta_0$ (from all FIE layers) will be learned with all the other parameters with back-propagation, including the parameters in the non-linear function and those in the last linear layer.

\begin{figure}[t]
\centering
     \begin{subfigure}
         \centering
         \includegraphics[width=0.23\textwidth]{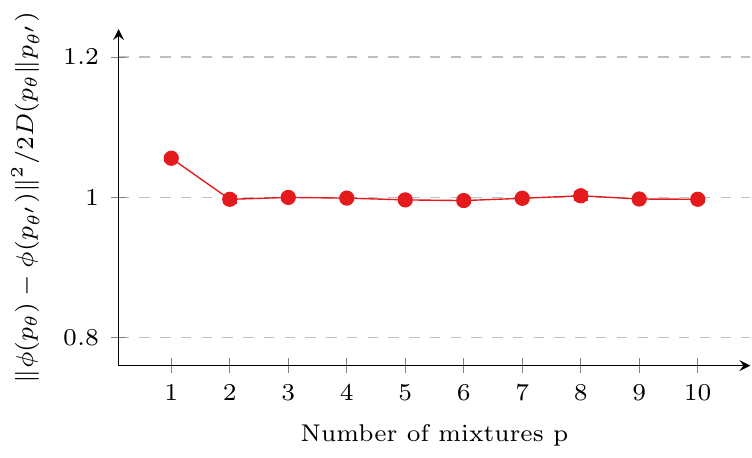}
     \end{subfigure}
     \hfill
     \begin{subfigure}
         \centering
         \includegraphics[width=0.23\textwidth]{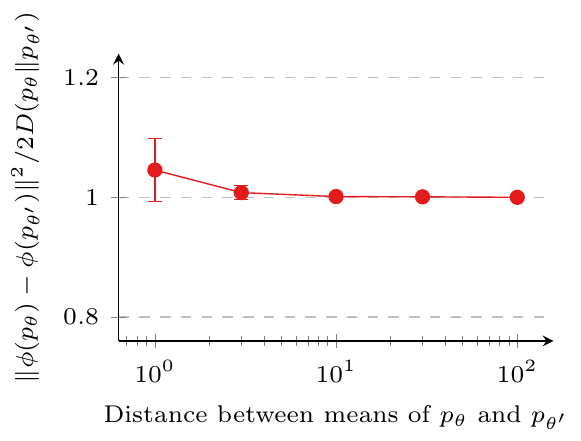}
     \end{subfigure}
     \hfill
     \caption{Simulation study on the ratio between the KL divergence of two distributions and the squared distance between their embeddings, varying some parameters of the underlying distributions.} \label{fig:sim1}
\end{figure}

\begin{figure}[t]
\centering
     \begin{subfigure}
         \centering
         \includegraphics[width=0.23\textwidth]{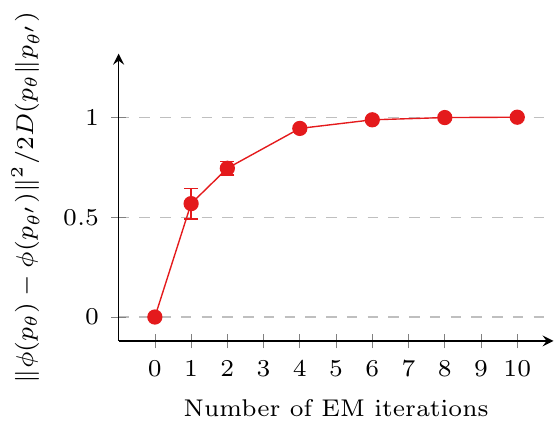}
     \end{subfigure}
     \hfill
     \begin{subfigure}
         \centering
         \includegraphics[width=0.23\textwidth]{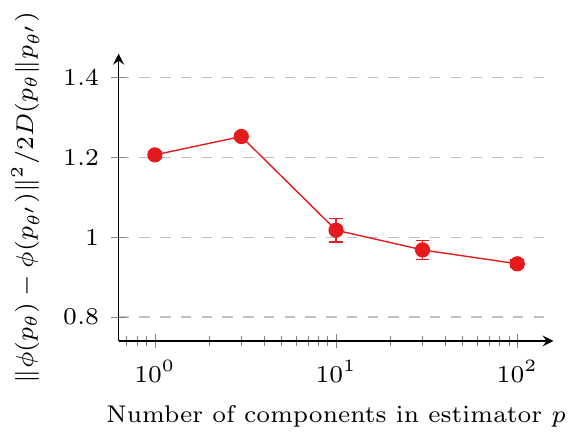}
     \end{subfigure}
     \caption{Simulation study on the ratio between the KL divergence between two distributions and the squared distance between their embeddings, varying some parameters of embedding method.} \label{fig:sim2}
\end{figure}

\section{Experiments}\label{sec:experiments}
In this section, we first validate our theoretical findings on synthetic datasets. Then, we compare FIE to existing GNNs on real-world datasets and show that FIE can achieve comparable performance to GATs. Additionally, we show that unsupervised FIE embedding combined with a state-of-the-art boosting method can achieve significant improvement compared to existing unsupervised node embedding methods and similar performance to its supervised counterpart. 

\begin{table*}[ht]
    \centering
    \caption{Classification accuracy on real world datasets. The reported values are taken from the literature (denoted with an asterisk) or are obtained from our reimplementations. In this latter case, we report the mean and the standard deviation over 10 runs. }
    \label{tab:classification}
    \resizebox{.85\textwidth}{!}{
    \begin{tabular}{lcccccc}\toprule
        & \multicolumn{3}{c}{Semi-supervised (transductive) learning tasks} & \multicolumn{3}{c}{Supervised learning tasks} \\
        Method & Cora & Citeseer & Pubmed & Reddit & ogbn-arxiv & ogbn-products \\
        \midrule
        WWL & $76.60 \pm 0.00$ & $66.60  \pm 0.00$ & $78.70 \pm 0.00$ & $ 96.13 \pm 0.04 $ & $64.67 \pm 0.08$ & $62.53 \pm 0.11$ \\ %
        Node2vec & $73.55 \pm 0.07$ & $54.01 \pm 1.16$ & $71.22 \pm 1.04$ & $95.58 \pm 0.08$ & $70.67 \pm 0.09$ & $74.77 \pm 0.14$  \\ %
        \midrule
        FIE unsup & $\mathbf{82.36\pm0.28}$ & $72.02\pm1.34$ & $\mathbf{79.22\pm0.17}$ & $\mathbf{96.61\pm0.02}$  & $72.17\pm0.07$ & $79.24\pm0.11$ \\
        \midrule\midrule
        MLP & $55.1$$^{\rm *}$ & $46.5$$^{\rm *}$ & $71.4$$^{\rm *}$ & - & $55.50 \pm 0.23$$^{\rm *}$ & $61.06 \pm 0.08$$^{\rm *}$ \\ %
        Label Propagation & $68.0$$^{\rm *}$ & $45.3$$^{\rm *}$ & $63.0$$^{\rm *}$ & - & $68.32 \pm 0.00$$^{\rm *}$ & $74.34 \pm 0.00$$^{\rm *}$  \\ %
        GCN & $81.4 \pm 0.50$$^{\rm *}$ & $70.9 \pm 0.50$$^{\rm *}$ & $79.0 \pm 0.3$$^{\rm *}$  &  - & $71.74\pm0.29$$^{\rm *}$ & $78.97\pm0.33$$^{\rm *}$ \\ %
        GraphSAGE & - & - & - & $94.32$$^{\rm *}$ & $71.49\pm0.27$$^{\rm *}$ & $78.70\pm0.36$$^{\rm *}$ \\ %
        GAT & $81.59 \pm 0.67$ & $70.08 \pm 0.58$ & $79.14 \pm 0.92$ &  $96.37  \pm0.18$ &  $71.59\pm0.38$$^{\rm *}$  & $79.04\pm1.54$$^{\rm *}$ \\ %
        GATv2 & $82.04 \pm 0.82$ & $69.72 \pm 1.03$ & $75.66 \pm 0.93$ & $96.58 \pm 0.01$ & $71.87\pm0.25$$^{\rm *}$ & $\mathbf{80.63\pm0.70}$$^{\rm *}$ \\ %
        \midrule
        FIE sup & $81.79\pm0.99$ & $\mathbf{72.51\pm0.34}$ & $78.35\pm0.26$ & $96.25 \pm 0.06$ & $\mathbf{72.39\pm0.21}$ & $79.25\pm0.25$ \\ \bottomrule
    \end{tabular}
    }
\end{table*}

\subsection{Evaluation on synthetic datasets} \label{sec:simstudy}
We provide empirical evidence that substantiates the results of Theorem~\ref{thm:closeness}. 
In the case of Gaussian mixtures, we show that the $\ell_2$-distance in the FIE space approximates well the KL divergence between the underlying distributions. We study the approximation quality under various configurations of the distributions and embedding hyperparameters. 

\paragraph{Experimental setup}
We conduct a simulation study on synthetic data. We generate samples from two Gaussian mixture distributions, $p_{\theta}$ and $p_{\theta'}$, and use the Goldberg approximation~\citep{goldberger2003efficient} to compute an approximation of the KL divergence between the two Gaussian mixtures as the ground truth. Using these generated samples, we compute the FIEs for the two distributions, as described in Eq.~\eqref{eq:gmm_fie}. We then analyze the relationship between the $\ell_2$-distance of the two embeddings and $2 \mathrm{KL}(p_{\theta} | p_{\theta'})$ as we vary the parameters of the distributions or the embedding hyperparameters.

\paragraph{Sensitivity to underlying distributions}
Here, we keep the hyperparameters of FIE fixed and only vary the underlying distributions to understand under which conditions the FIE accurately approximates the KL divergence. We generate distributions $p_{\theta}$ and $p_{\theta'}$ as mixtures of three Gaussians with identity covariance. We vary respectively (i) the number of components, (ii) the distance between the mean of $p_{\theta}$ and of $p_{\theta'}$, and (iii) the distance among the components in $p_{\theta'}$. As shown in Figure~\ref{fig:sim1}, the squared embedding distances closely match the KL divergence $2 \mathrm{KL}(p_{\theta} | p_{\theta'})$, as predicted by Theorem~\ref{thm:closeness}.

\paragraph{Sensitivity to embedding hyperparameters}
In this set of experiments, we fix $p_{\theta}$ and $p_{\theta'}$ and vary the parameters of our embedding method. Specifically, we alter the number of components, $p$, in the Gaussian mixtures between 1 and 100, and the number of EM iterations, $M$, in the maximum likelihood estimation between 0 and 10. The results in Figure~\ref{fig:sim2} demonstrate that the squared distance between the embeddings approaches the KL divergence as the number of EM iterations increases, as the estimation error decreases. Additionally, we observe that the quality of the approximation improves as the number of components in the Gaussian mixtures increases.

\begin{figure*}[ht]
    \centering
    \includegraphics[width=.25\textwidth]{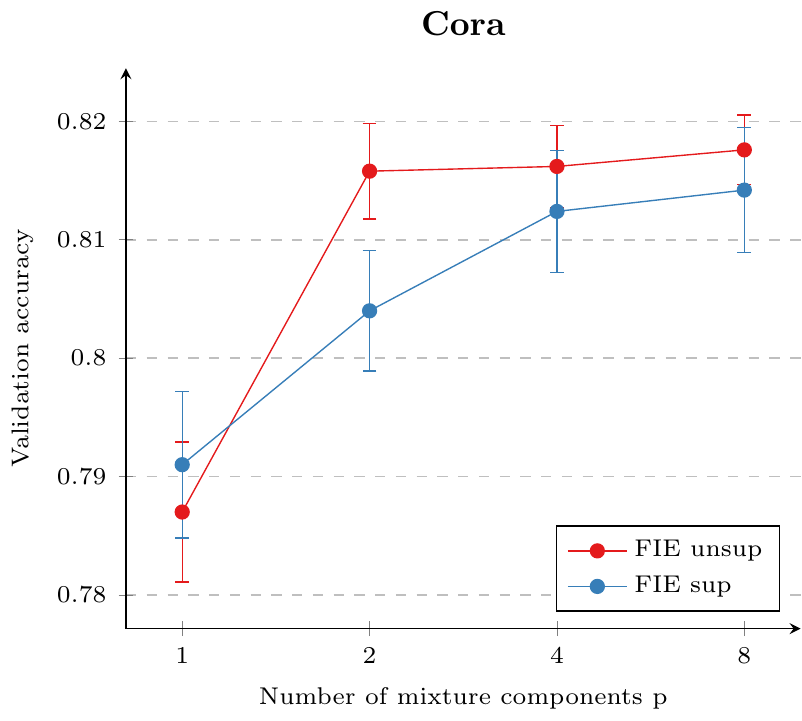}
    \includegraphics[width=.25\textwidth]{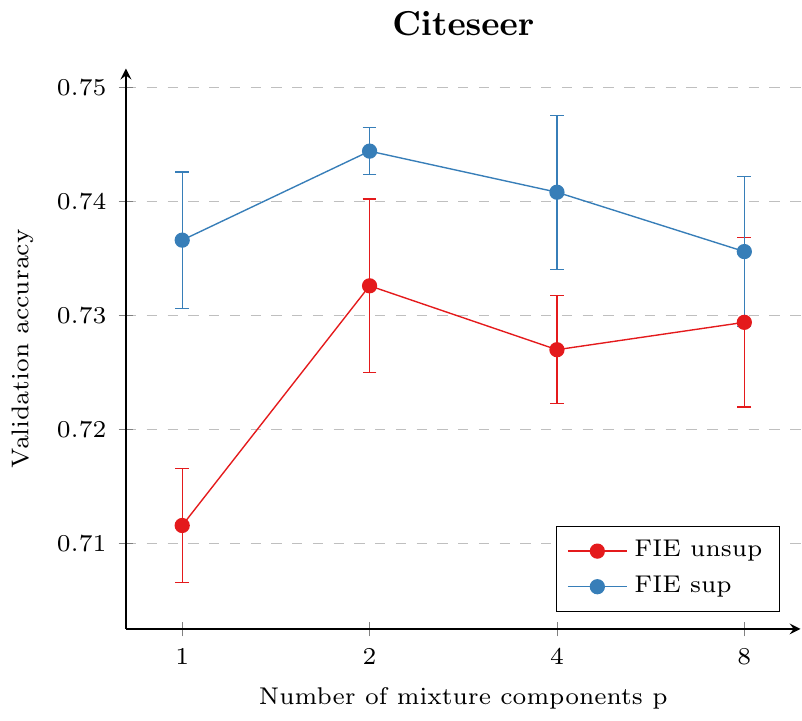}
    \includegraphics[width=.25\textwidth]{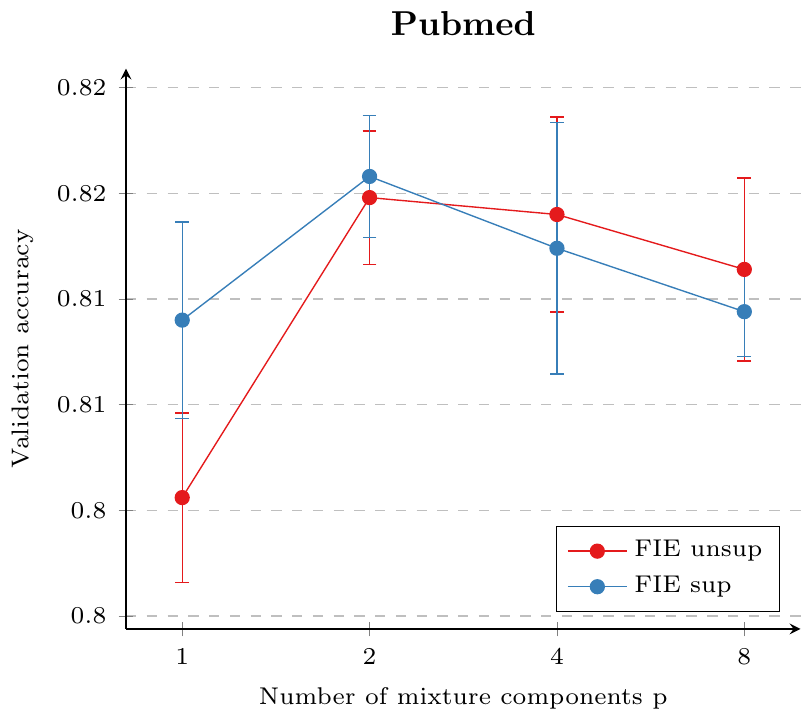}
    \caption{Effect of number of mixture components $p$ on the validation accuracy for semi-supervised learning datasets.}
    \label{fig:num_mixtures}
\end{figure*}

\subsection{Evaluation on real-world datasets}
We evaluate FIE and compare its variants to state-of-the-art attention-based GNNs and unsupervised node embedding methods on several real-world datasets for (semi-) supervised node classification.

\paragraph{Datasets and experimental setup}
We assess the performance of our method with six widely used benchmark datasets for node classification, including Cora, Citeseer, Pubmed~\citep{sen2008collective} as semi-supervised transductive learning datasets and Reddit~\citep{hamilton2017inductive}, ogbn-arxiv~\citep{hu2020ogb1}, ogbn-products~\citep{hu2020ogb1} as medium- or large-scale supervised learning datasets. 

We compare our method to both unsupervised and supervised methods for node representation learning. Unsupervised node embedding methods include WWL~\citep{togninalli2019wasserstein}, which use simple average aggregation of neighborhoods, and Node2vec~\citep{grover2016node2vec}. The supervised comparison partners include MLP~\citep{hu2020ogb1}, Label propagation~\citep{zhu2002learning}, GCN~\citep{kipf2017semisupervised}, GraphSAGE~\citep{hamilton2017inductive}, GAT~\citep{velickovic2018graph}, and GATv2~\citep{brody2022how}.

All results for the comparison partners are either taken from the original paper (denoted with an asterisk in the table) or obtained by our reimplementation if not available. All results are computed from 10 runs using different random seeds with the optimal hyperparameters selected on the validation set. Full details on the datasets, experimental setup and implementation details can be found in the Appendix.

\begin{figure*}[ht]
    \centering
    \begin{subfigure}
         \centering
         \includegraphics[width=0.28\textwidth]{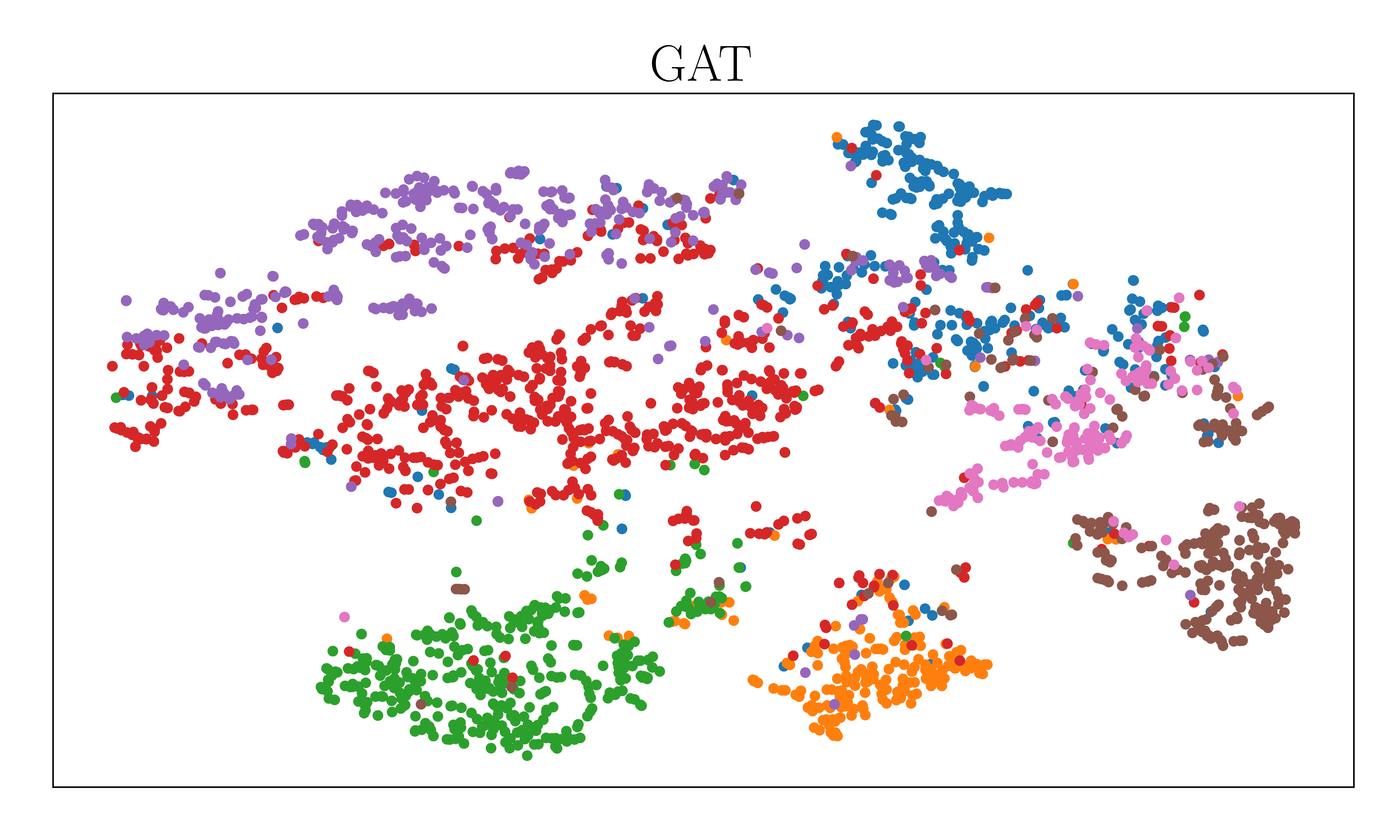}
     \end{subfigure}
     \hfill
     \begin{subfigure}
         \centering
         \includegraphics[width=0.28\textwidth]{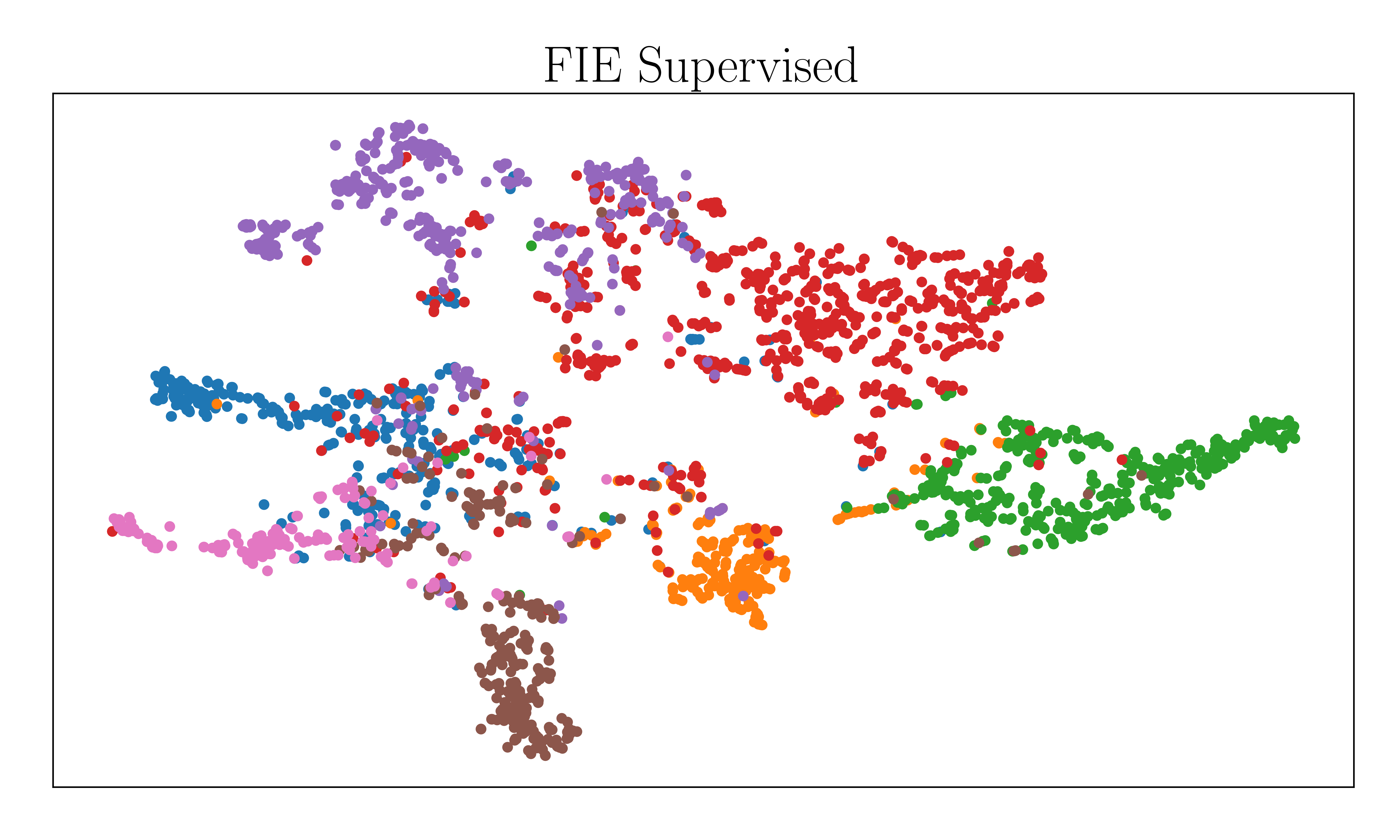}
     \end{subfigure}
     \hfill
     \begin{subfigure}
         \centering
         \includegraphics[width=0.28\textwidth]{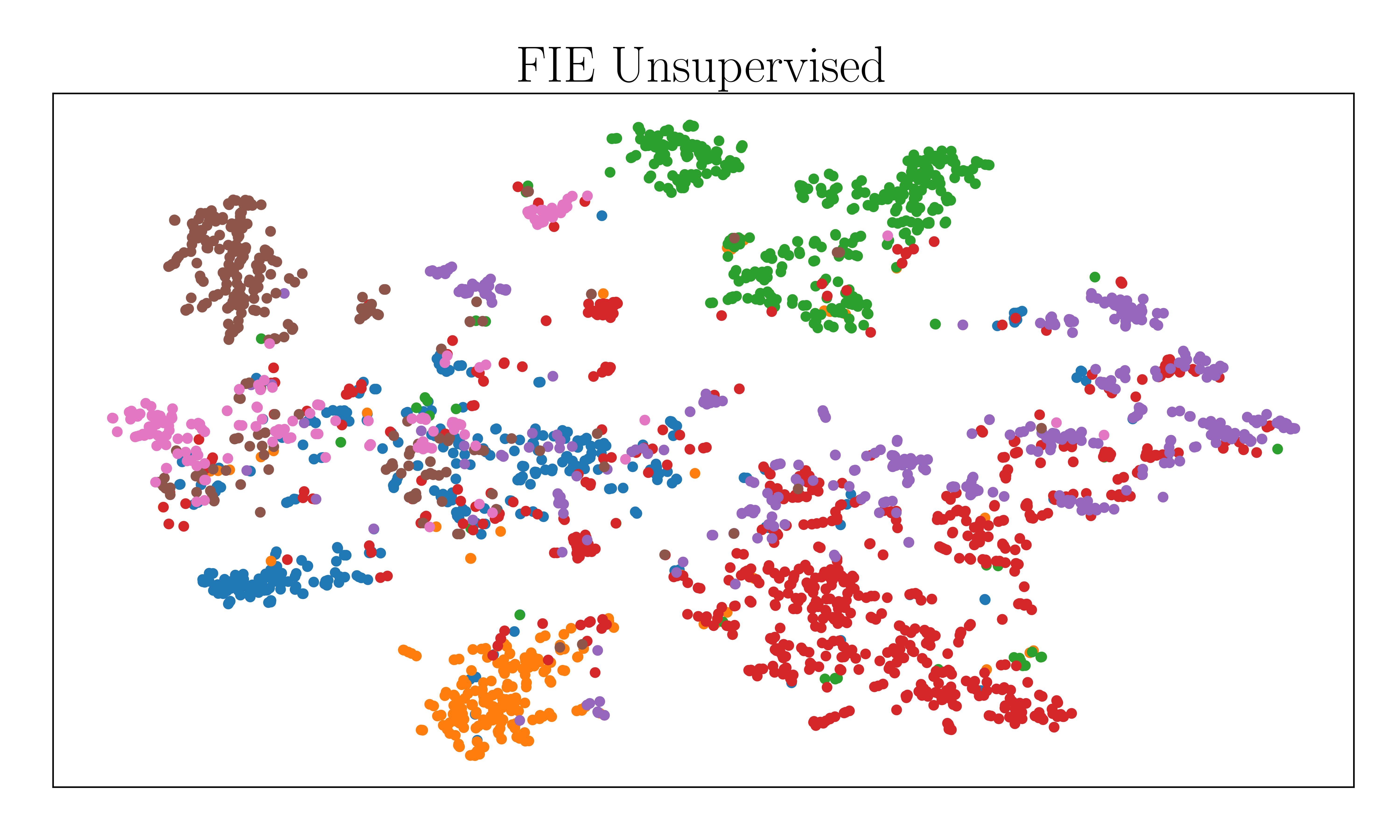}
     \end{subfigure}
    \caption{Left and center panels show the t-SNE visualization of the node embedding space on the last hidden layer (with hidden dimension 64) for GAT and supervised FIE on the Cora dataset. Right panel shows the t-SNE visualization of the unsupervised FIE (dimension 64) on the Cora dataset. The embedding colors represent the node labels, and are consistent throughout the plots.}
    \label{fig:tsne}
\end{figure*}

\paragraph{Unsupervised node embedding}
In Section~\ref{sec:learning}, we propose an unsupervised method for learning the anchor distribution parameter $\theta_0$ through a layer-by-layer k-means algorithm, similar to \citet{chen2020convolutional}. We sample a subset of 300,000 nodes from each layer and perform k-means on the subset to compute $\theta_0$. Once learned, the node embeddings for that layer can be computed. This method requires only \emph{one single forward pass} to obtain the final node embeddings. We then train a classifier using the concatenated node embeddings across all layers, as outlined in Eq.~\eqref{eq:multilayer_kernel}. For semi-supervised learning datasets, we use a logistic regression model, which is less prone to overfitting. For large supervised datasets, we use a LightGBM classifier~\citep{ke2017lightgbm}, the hyperparameters of which are automatically tuned by FLAML~\citep{wang2021flaml}. The same classifiers are used for other comparison partners.

Our results presented in Table~\ref{tab:classification} demonstrate that our unsupervised embedding (FIE unsup) consistently outperforms existing unsupervised methods by a large margin over all datasets. More remarkably, our unsupervised embedding achieves comparable or better performance compared to supervised methods, including GAT variants and its supervised counterpart, showing the effectiveness of our unsupervised learning strategy. The results suggest that learning a good classifier on a general-purpose node representation in a decoupled way can lead to strong performance.

\paragraph{Supervised node embedding}
In addition to our unsupervised method for learning the anchor distribution parameter $\theta_0$, we also propose a supervised approach for training the entire model end-to-end, as described in Section~\ref{sec:learning}. We utilize a cross entropy loss and the Adam optimizer~\citep{kingma2015adam} for optimization. Our unsupervised embedding method serves as a natural initialization for $\theta_0$ and we use it for semi-supervised learning datasets. 
Our results, shown in Table~\ref{tab:classification}, demonstrate that our supervised FIE method achieves comparable or better performance compared to GAT variants, without the needing complex tricks in the model architecture such as using Leaky ReLU.

\paragraph{Effect of number of mixture components}
In this study, we investigate the effect of the number of components $p$ used in the Gaussian mixture model on our FIE approach. Specifically, we aim to demonstrate that the aggregation mechanism in Eq.~\eqref{eq:gmm_fie} derived from our approach is superior to simple average pooling. Figure~\ref{fig:num_mixtures} illustrates the impact of varying the number of components on validation accuracy for both supervised and unsupervised variants across several datasets. As the number of components $p$ decreases to $1$, our FIE method approaches average pooling of the neighbors. However, as $p$ increases, our FIE behaves similarly to other multi-head attention-based message passing methods, as the concatenation of weighted averages of neighbors. Our results suggest that using $p>1$ leads to substantial improvements in performance, with the optimal value of $p$ varying depending on the dataset. For example, on Cora, $p=8$ produces the best results, while the best performance is attained at $p=2$ for Citeseer and Pubmed. This study illustrates the superiority of our attention mechanism to average pooling and suggests that the optimal value of $p$ should be selected through cross validation.%

\paragraph{t-SNE visualization of embedding spaces}
To further demonstrate the effectiveness of our FIE method qualitatively, we present a visualization of the node embedding space for both unsupervised and supervised training modes. Figure~\ref{fig:tsne} shows the t-SNE visualization of the node embeddings for the Cora dataset using both GAT and FIE. All methods produce well-separated clusters for different classes. The clusters in the FIE embeddings obtained via unsupervised training are less distinct, as expected since the method has no access to class labels. Despite this, the resulting embedding is able to achieve classification performance comparable to or even better than the other methods.

\section{Discussion}
We have proposed the Fisher information embedding model, a novel class of attention-based node embeddings that can be learned flexibly in either supervised or unsupervised settings. By leveraging tools from information geometry, our method constructs a new attention mechanism, offering deeper insights into the geometric aspects of attention-based models. Although supervised node embedding methods are widely used in practice, our unsupervised node embeddings demonstrate comparable performance. Our work is related to a recent class of graph models, namely graph transformers~\citep{ying2021transformers,mialon2020trainable,chen2022structure,rampavsek2022recipe}, which also use an attention-based operator~\citep{vaswani2017attention}, called self-attention, on the full set of nodes. However, rather than pooling the multiset, the self-attention returns another multiset of node features. Our work provides a first step towards understanding self-attention from an information geometry perspective.

A limitation of our approach is its reliance on Gaussian mixtures for parameter estimation using the EM algorithm. However, our framework has potential for incorporating a wider range of distribution families, which could lead to enhanced results.
Exploring these extensions, alongside investigation into the generalization bounds of the FIE model, presents a promising direction for future research.

\section*{Acknowledgements}
The authors would like to thank Dr.\ Carlos Oliver, Dr. Armin Lambacher and the reviewers for their insightful feedback. %

\bibliography{mybib}
\bibliographystyle{icml2023}

\newpage
\appendix
\onecolumn

\vspace*{0.3cm}
\begin{center}
    {\huge Appendix}
\end{center}
\vspace*{0.5cm}

This appendix provides both theoretical and experimental materials and is organized as follows: Section~\ref{sec:supp_background} provides a more detailed background on graph attention networks. Section~\ref{sec:supp_proofs} presents proofs for all the Lemmas and theorems. Section~\ref{sec:supp_estep} provides variations for the E-step of FIE on the manifold of Gaussian mixtures. Section~\ref{sec:supp_exp} provides experimental details and additional results.

\section{Background on graph attention networks}\label{sec:supp_background}
GATs \cite{velickovic2018graph} are GNNs that have an attention mechanism to weight the contributions of each of the nodes when aggregating features from a neighborhood in the message passing framework. 
Indeed, the GAT model uses a score function $e(h_i, h_j) = {\rm LeakyReLU}(a^\top [W h_i \| W h_j] )$ to decide the importance of the features $h_j$ of a neighbor $j$ of node $i$. 
$a$ and $W$ are learnable parameters. 
The attention score for edge $(i,j)$ is then computed as $\alpha_{ij} = {\rm softmax}_{j}(e(h_i, h_j))$.

Then, GAT aggregates the features from the neighborhood $\mathcal{N}_i$ of $i$ by computing the weighted sum $$h'_i = \sigma \left(  \sum_{j \in \mathcal{N}_{i}} \alpha_{ij} W h_j  \right)$$

Subsequently, \cite{brody2022how} proposed a different attention mechanism, dubbed GATv2, that computes the attention scores as $e(i, j) = a^\top {\rm LeakyReLU}( W [ h_i \| h_j] )$, which increases the representative power of the model.

\section{Proofs}\label{sec:supp_proofs}
\subsection{Proof of Lemma~\ref{lemma:expressive}}
\begin{proof}
We follow the arguments by~\citet[Theorem 3]{xu2018powerful}. The WL test applies a predetermined injective hash function $g$ to
update the WL node labels $a^{(t)}(v)$:
\begin{equation*}
    a_t(v)=g\left(\left\{\left(a_{t-1}(v),a_{t-1}(u)\right):u\in\Ncal(v)\right\}\right).
\end{equation*}
And our kernel embedding at iteration $t$ in the example of neighborhood multisets is given by
\begin{equation*}
    \psi_t(v)=\psi_{\mathrm{multiset}}^{(t)}(\left\{(\psi_{t-1}(v),\psi_{t-1}(u)):u\in\Ncal(v)\right\}).
\end{equation*}
We show by induction that, for any iterations $t$, there always exists an injective function $\varphi$ such that $\psi_t(v)=\varphi(a_t(v))$. This apparently holds for $t=0$ as $\psi_0=a_0=a$. Now let us assume that this condition holds for $t-1$, we will show it also holds for $t$. Substituting $\psi_{t-1}(v)$ with $\varphi(a_{t-1}(v))$ gives us:
\begin{equation*}
     \psi_t(v)=\psi_{\mathrm{multiset}}^{(t)}(\left\{(\varphi(a_{t-1}(v)),\varphi(a_{t-1}(u))):u\in\Ncal(v)\right\}).
\end{equation*}
Since the composition of injective functions is injective, we have $\psi=\psi_{\mathrm{multiset}}^{(t)}\circ\varphi$ is injective, where by abuse of notation, $\varphi$ is applied element-wise to the multiset. Then, we have
\begin{equation*}
    \psi_t(v)=\psi\circ g^{-1}(a_t(v)),
\end{equation*}
such that $\psi\circ g^{-1}$ is injective, as being the composition of injective functions. Therefore, we conclude the lemma.
\end{proof}

\subsection{Proof of Lemma~\ref{lemma:generalization}}
This is a classical result, and its proof can be found in, \eg \citet{boucheron2005theory}.

\subsection{Proof of Theorem~\ref{thm:information_geometry}}
\begin{proof}
The proof is adapted from~\citet[Theorem 3.20]{amari2000methods}. Let $(g,\nabla,\nabla^*)$ be a dualistic structure induced by the divergence $D$ as shown in~\citep{amari2000methods}. We denote by $\exp_{\mu}$ and $\exp_{\mu}^*$ respectively the exponential maps for $\nabla$ and $\nabla^*$. Theorem 3.20 showed that
	\begin{equation}
		D(u\|\mu)+D(\mu\| v)-D(u\| v)=\langle \exp_{\mu}^{-1}(u),\exp_{\mu}^{*-1}(v)\rangle_g +o(\Delta^3).
	\end{equation}
	We also have $\|R_{\mu}^{-1}(u)-\exp_{\mu}^{-1}(u)\|_g=O(\Delta^2)$ and $\|R_{\mu}^{-1}(v)-\exp_{\mu}^{*-1}(v)\|_{g}=O(\Delta^2)$ since $R_{\mu}^{-1}(u)=u-\mu$. Finally, we have
	\begin{equation*}
	\begin{aligned}
	    & \left|\langle R_{\mu}^{-1}(u), R_{\mu}^{-1}(v)\rangle_{g} - \langle \exp_{\mu}^{-1}(u),\exp_{\mu}^{*-1}(v)\rangle_g\right| \\
	    =& \left|\langle R_{\mu}^{-1}(u)-\exp_{\mu}^{-1}(u), R_{\mu}^{-1}(v)\rangle_{g} - \langle \exp_{\mu}^{-1}(u),\exp_{\mu}^{*-1}(v)-R_{\mu}^{-1}(v)\rangle_g\right| \\
	    \leq & \left|\langle R_{\mu}^{-1}(u)-\exp_{\mu}^{-1}(u), R_{\mu}^{-1}(v)\rangle_{g}\right| + \left|\langle \exp_{\mu}^{-1}(u),\exp_{\mu}^{*-1}(v)-R_{\mu}^{-1}(v)\rangle_g\right| \\
	    \leq & \|R_{\mu}^{-1}(v)\|_g\|R_{\mu}^{-1}(u)-\exp_{\mu}^{-1}(u)\|_g + \|\exp_{\mu}^{-1}(u)\|_g\|R_{\mu}^{-1}(v)-\exp_{\mu}^{*-1}(v)\|_{g}=O(\Delta^2),
	\end{aligned}
	\end{equation*}
	where the first inequality uses triangle inequality and the second inequality uses the Cauchy-Schwarz inequality.
\end{proof}

\subsection{Proof of Lemma~\ref{lemma:kl_div}}
Fisher information is the second derivative of KL divergence, which is a commonly known result in information theory. A proof can be found in \eg \citet[Page 87]{gourieroux1995statistics}.

\subsection{Proof of Theorem~\ref{thm:closeness}}
\begin{proof}
Following the definition in Eq.~\eqref{eq:fie_proba}, we have
\[
\begin{aligned}
\frac{\norm{\varphi_{\theta_0}(p_{\theta})-\varphi_{\theta_0}(p_{\theta'})}^2}{2}&=\frac{\norm{I(\theta_0)^{\nicefrac{1}{2} }(\theta-\theta')}^2}{2} \\
&=\frac{1}{2}(\theta-\theta')^{\top} I(\theta_0)(\theta-\theta').
\end{aligned}
\]
When both $\theta$ and $\theta'$ approach $\theta_0$, Lemma~\ref{lemma:kl_div} suggests that
\[
\mathrm{KL}(p_{\theta}\|p_{\theta_0})=\frac{1}{2}(\theta-\theta_0)^{\top} I(\theta)(\theta-\theta_0)+o\|\theta-\theta_0\|^2,
\]
and
\[
\mathrm{KL}(p_{\theta_0}\|p_{\theta'})=\frac{1}{2}(\theta'-\theta_0)^{\top} I(\theta_0)(\theta'-\theta_0)+o\|\theta'-\theta_0\|^2.
\]
We also have
\[
\left|\frac{1}{2}(\theta-\theta_0)^{\top} I(\theta)(\theta-\theta_0) - \frac{1}{2}(\theta-\theta_0)^{\top} I(\theta_0)(\theta-\theta_0)\right|=\frac{1}{2}\|I(\theta)-I(\theta_0)\|_2\|\theta-\theta_0\|_2^2,
\]
thus
\[
\mathrm{KL}(p_{\theta}\|p_{\theta_0})=\frac{1}{2}(\theta-\theta_0)^{\top} I(\theta_0)(\theta-\theta_0)+O\|\theta-\theta_0\|^2,
\]
By substituting $\mu$ by $p_{\theta_0}$, $u$ by $p_{\theta}$ and $v$ by $\theta'$ in Theorem~\ref{thm:information_geometry}, we have
\[
\mathrm{KL}(p_{\theta}\|p_{\theta_0})+\mathrm{KL}(p_{\theta_0}\|p_{\theta'})-\mathrm{KL}(p_{\theta}\|p_{\theta'})=(\theta-\theta_0)^{\top} I(\theta_0)(\theta'-\theta_0)+O(\Delta^2),
\]
where $\Delta=\max\{\|\theta-\theta_0\|,\|\theta'-\theta_0\|\}$. Finally, by substituting the KL divergence terms, we obtain the expected expansion.

\end{proof}

\subsection{Proof of Theorem~\ref{thm:injectivity} and Theorem~\ref{thm:lipschitz}}
The proof of two theorems are straightforward from the definition of $\varphi_{\theta_0}$ in Eq.~\eqref{eq:fie_proba}.

\subsection{Relationship between $\theta_{\mathrm{ML}}$ and $\theta_{\mathrm{ML}_{\theta_0}}$}
Here, we show that the estimator $\theta_{\mathrm{ML}_{\theta_0}}$ is a good proxy of the ML estimator.
\begin{lemma}
Let
    \begin{equation}
	\theta_{\mathrm{ML}_{\theta_0}}(\x):= \argmax_{\theta\in\Mcal}\E_{z|x,\theta_0}[\log p_{\theta}(x,z)],
\end{equation}
where $p_{\x}$ denotes the true density of $\x$. 
This is a good proxy for the maximum likelihood estimation of $\theta$.
\end{lemma}
\begin{proof}
We have that
\begin{equation}
    \argmax_{\theta\in\Mcal}\E_{z|x,\theta_0}[\log p_{\theta}(x,z)]  = \argmax_{\theta\in\Mcal}\E_{z|x,\theta_0}[\log p_{\theta}(x,z) - \log p_{\theta_0}(z|x)],
\end{equation}
as the term $\log p_{\theta_0}(z|x)$ does not depend on $\theta$.
Thanks to Jensen's inequality we obtain that 
\begin{equation}
    \E_{z|x,\theta_0}\left[\log \frac{p_{\theta}(x,z)}{p_{\theta_0}(z|x)}\right] \leq 
    \log \left( \E_{z|x,\theta_0}\left[ \frac{p_{\theta}(x, z)}{p_{\theta_0}(z|x)} \right]\right) =
    \log p_{\theta}(x).
\end{equation}
Moreover, we have that the difference is bounded by
\begin{equation}
\begin{split}
    \log p_{\theta}(x) - \E_{z|x,\theta_0}\left[\log \frac{p_{\theta}(x,z)}{p_{\theta_0}(z|x)}\right] &=
    \int \left(\log p_{\theta}(x) - \log\frac{p_{\theta}(x,z)}{p_{\theta_0}(z|x)} \right) p_{\theta_0}(z|x) \,dz \\
    &= \int \left( \log\frac{p_{\theta_0}(z|x)}{p_{\theta}(z|x)} \right) p_{\theta_0}(z|x) \,dz \\
    &= KL\left(p_{\theta_0}(z|x) \ \| \ p_{\theta}(z|x) \right)
\end{split}
\end{equation}
\end{proof}

\subsection{Particular case for the manifold of a single Gaussian}
\begin{lemma} 
Consider as the family of distributions the family of Gaussians $\Ncal(\mu, I)$, parametrized by $\theta = \mu$.
Then, we have $D(p_\theta\| p_{\theta'}) = \frac{\norm{\psi(\theta) - \psi(\theta')}^2}{2}$.
\end{lemma}
\begin{proof}
We have $D(p_\theta \| p_\theta') = \frac{\norm{\mu - \mu'}^2}{2}$.
Moreover, we have that the Fisher information matrix is the identity, i.e. $I(\theta) = I$.
Then $\psi(\theta) = \mu$, and $\norm{\psi(\theta) - \psi(\theta')}^2 = \norm{\mu - \mu'}^2$, so $D(p_\theta \| p_\theta') = \norm{\psi(\theta) - \psi(\theta')}^2$.
\end{proof}

\section{Variations on the E step for Fisher information embedding}\label{sec:supp_estep}
In particular, the E-step can be solved in several ways, both analytically and not, depending on the problem at hand:
\begin{itemize}
    \item Without further constraints on $\boldsymbol{\alpha}$, one can show that the optimal $a_{ij}$'s are
    \begin{equation*}
        \alpha_{ij}=p_{\theta}(z_i=j|x_i)=\frac{\Ncal(x_i,w_j,\Sigma_j)}{\sum_{l=1}^p \Ncal(x_i,w_l,\Sigma_l)}.
    \end{equation*}
    \item If one adds an hard constraint of the form $\sum_i a_{ij} = \nicefrac{n}{p}$, e.g. to avoid collapsed solutions, then the minimization problem %
    is an entropy-regularized optimal transport problem, which can be solved efficiently by the Sinkhorn–Knopp algorithm.
    Indeed, the problem can be formulated as 
    \begin{equation*}
        \min_{\boldsymbol{\alpha}\in\Pi'} \sum_{i=1}^n \sum_{j=1}^p \alpha_{ij}C_{ij} - H(\boldsymbol{\alpha}),
    \end{equation*}
    with $C_{ij} = w_i \log p_{\theta}(x_i,z_i=j)$ and $\Pi':=\{\alpha_{ij}\geq 0,\sum_{j}\alpha_{ij}=1, \sum_i a_{ij} = \nicefrac{n}{p}\}$.
    \item In many situations, adding a hard constraint on the rows of $\boldsymbol{\alpha}$ is too strong of a requirement. Instead, one can add a regularization term in the objective function to skew the optimal solution towards the desired marginal. The E-step can be then formulated as an unbalanced optimal transport (UOT) problem \cite{} as follows:
    \begin{equation*}
    \begin{split}
        \min_{\boldsymbol{\alpha}\in \mathbb{R}^{n\times p}_{+}} \sum_{i=1}^n \sum_{j=1}^p & \alpha_{ij}C_{ij} + H(\boldsymbol{\alpha}) + \tau_1 \KL(\boldsymbol{\alpha}\boldsymbol{1}_n \| \boldsymbol{1}_p ) + \tau_2 \KL(\boldsymbol{\alpha}^\top\boldsymbol{1}_p \| \nicefrac{n}{p} \boldsymbol{1}_n ),
    \end{split}
    \end{equation*}
    with $\tau_1 \rightarrow \infty$ to enforce that $\sum_j a_{ij}=1$. Similarly to OT, this problem can be solved efficiently with a variant of the Sinkhorn algorithm.
\end{itemize}

\section{Experimental details and additional results}\label{sec:supp_exp}
Here, we provide experimental details and additional experimental results.

\subsection{Evaluation on synthetic datasets}
\begin{figure*}[t]
\centering
     \begin{subfigure}
         \centering
         \includegraphics[width=0.32\textwidth]{sim/exp0.pdf}
     \end{subfigure}
     \hfill
     \begin{subfigure}
         \centering
         \includegraphics[width=0.32\textwidth]{sim/exp1.pdf}
     \end{subfigure}
     \hfill
     \begin{subfigure}
         \centering
         \includegraphics[width=0.32\textwidth]{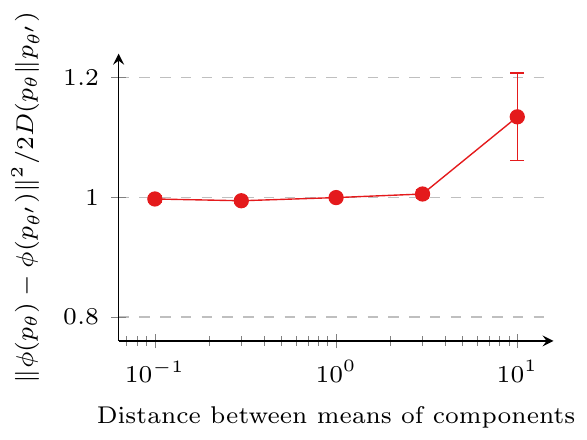}
     \end{subfigure}
     \caption{Simulation study on the ratio between the KL divergence between two distributions and the squared distance between their embeddings, varying some parameters of the underlying distributions.} \label{fig:appendix:sim1}
\end{figure*}

In the first set of experiments, we keep the parameters of the embedding method fixed and we vary the underlying distributions, to understand in which cases the fisher embedding approximates well the KL divergence. 
In particular, we use as the family of distributions a family of Gaussian mixtures with $k =10$ components, and with each component having identity covariance. 
Moreover, we fix the number of EM iterations to $T=10$. 

We then vary the underlying data-generating distributions as follows.
\begin{itemize}
\item
  In the first simulation we have $p_1$ as a mixture of three
  Gaussians $\mathcal{N}(\mu_{1,j}, I)$, with
  $\mu_{1,1} = (-5, -2)$, $\mu_{1,1} = (-5, 0)$,
  $\mu_{1,1} = (-5, 2)$. Then we vary $p_2$ as a mixture of $\kappa$
  Gaussians $\mathcal{N}(\mu_{2,j}, I)$, with means evenly spaced
  between $(5,-2)$ and $(5, 2)$, for $\kappa \in [1, 10]$.
\item
  In the second simulation, we have that both the distributions are
  mixtures of three Gaussians $\mathcal{N}(\mu_{i,j}, I)$, and we vary
  the distance between the means of the two distributions, keeping the
  distance between components of the same distribution fixed. In
  particular, we let $\mu_{1,1} = (-5, -2)$, $\mu_{1,1} = (-5, 0)$,
  $\mu_{1,1} = (-5, 2)$ and $\mu_{2,1} = (-5+d, -2)$,
  $\mu_{2,1} = (-5+d, 0)$, $\mu_{2,1} = (-5+d, 2)$, for
  $d \in [1, 100]$.
\item
  In the third simulation, we again have that both the distributions are
  mixtures of three Gaussians $\mathcal{N}(\mu_{i,j}, I)$, but this
  time we vary the distance between the components of each distribution,
  keeping the distance between the means of the two distributions fixed.
  In paricular, we let $\mu_{1,1} = (-5, -d)$,
  $\mu_{1,1} = (-5, 0)$, $\mu_{1,1} = (-5, d)$ and
  $\mu_{2,1} = (5, -d)$, $\mu_{2,1} = (5, 0)$,
  $\mu_{2,1} = (5, d)$, for $d \in [0.1, 10]$.
\end{itemize}

In the second set of experiments we keep the underlying distributions fixed and the vary the parameters of our embedding method. 
In particular, we let both distributions be a
  mixture of three Gaussians $\mathcal{N}(\mu_{1,j}, I)$ 
  with $\mu_{1,1} = (-5, -2)$,
  $\mu_{1,1} = (-5, 0)$, $\mu_{1,1} = (-5, 2)$ and
  $\mu_{2,1} = (5, -2)$, $\mu_{2,1} = (5, 0)$,
  $\mu_{2,1} = (5, 2)$.

We then vary the number $k$ of components in the parametric family of Gaussian mixtures for the maximum likelihood estimation in $[1, 100]$, and the number $T$ of EM iterations in $[0,10]$. 

\subsection{Evaluation on real-world datasets}

\subsubsection{Computation Details}
All experiments were performed on a shared GPU and CPU cluster equipped with GTX1080 and TITAN RTX. About 20 of these GPUs were used simultaneously, and the total computational cost of this research project was about 500 GPU hours.

\subsubsection{Datasets}
The datasets that we use for classification are classical ones. The statistics for each dataset is summarized in Table~\ref{app:tab:dataset}

\begin{table}[h]
    \centering
    \begin{tabular}{lccc}\toprule
         Dataset & Number of nodes & Number of edges & Number of classes \\ \midrule
         Cora & 2708 & 5429 & 7  \\ 
         Citeseer & 3327 & 4732 & 6 \\
         Pubmed & 19717 & 44338 & 3 \\
         Reddit & 232965 & 114615892 & 41\\
         ogbn-arxiv & 169343 & 1166243 & 40 \\
         ogbn-products & 2449029 & 61859140 & 47\\
         \bottomrule
    \end{tabular}
    \caption{Summary of considered datasets.}
    \label{app:tab:dataset}
\end{table}

\subsubsection{Baseline results}
We report the results from the following papers.
For MLP, for the Cora, Citeseer and Pubmed datasets we report the results from \cite{velickovic2018graph}, for the ogbn-arxiv and ogbn-products datasets we report the results from \cite{hu2020ogb1}.
For Label Propagation, for the Cora, Citeseer and Pubmed datasets we report the results from \cite{velickovic2018graph}, for the ogbn-arxiv and ogbn-products datasets we report the results from \cite{hu2020ogb1}.
For CGN, for the Cora, Citeseer and Pubmed datasets we report the results from \cite{velickovic2018graph} and for the ogbn-arxiv and ogbn-products datasets we report the results from \cite{brody2022how}.
For GraphSAGE, for the Reddit dataset we report the results from \cite{hamilton2017inductive} and for the ogbn-arxiv and ogbn-products datasets we report the results from \cite{brody2022how}.
For GAT and GATv2, for the ogbn-arxiv and ogbn-products datasets we report the results from \cite{brody2022how} and for other datasets we report the results from our reimplementation.

\subsubsection{Hyperparameter Choices and Reproducibility}
\paragraph{Hyperparameter choice.}
In general, we perform a very limited hyperparameter search to produce the results in Table~\ref{tab:classification}. The hyperparameters for training FIE models on different datasets are summarized in Table~\ref{tab:supp_hyperparameter_unsup} and Table~\ref{tab:supp_hyperparameter_sup}, respectively for unsupervised and supervised modes of FIE. For supervised learning tasks, a dropout with rate equal to 0.5 is used for training supervised embeddings of FIE. For large supervised datasets (Reddit, ogbn-arxiv, and ogbn-products), we use a LightGBM classifier~\citep{ke2017lightgbm}, the hyperparameters of which are automatically tuned by FLAML~\citep{wang2021flaml}.

\paragraph{Optimization.}
All our models are trained with the Adam optimizer~\citep{kingma2015adam} with a fixed learning rate equal to 0.001.

\begin{table}[h]
    \centering
    \begin{tabular}{lcc}\toprule
         Hyperparameter & \{Cora, Citeseer, Pubmed\} &  \{Reddit, ogbn-arxiv, ogbn-products\} \\ \midrule
         Number of layers & [2,3,4] & [3, 4, 5] \\
         Hidden dimensions & [128, 256, 512] & [128, 256, 512] \\
         Number of mixture components & [1, 2, 4, 8] & [1, 2, 4, 8]
         \\ \bottomrule
    \end{tabular}
    \caption{Hyperparameters for unsupervised mode of FIE.}
    \label{tab:supp_hyperparameter_unsup}
\end{table}

\begin{table}[h]
    \centering
    \begin{tabular}{lcc}\toprule
         Hyperparameter & \{Cora, Citeseer, Pubmed\} &  \{Reddit, ogbn-arxiv, ogbn-products\} \\ \midrule
         Number of layers & [2,3,4] & [3, 4, 5] \\
         Hidden dimensions & [16, 32, 64] & [128, 256] \\
         Number of mixture components & [1, 2, 4, 8] & [1, 2, 4, 8]
         \\ \bottomrule
    \end{tabular}
    \caption{Hyperparameters for supervised mode of FIE.}
    \label{tab:supp_hyperparameter_sup}
\end{table}

\end{document}